\documentclass[10pt,twocolumn,a4paper]{article}

\usepackage[a4paper, top=2.5cm, bottom=2.5cm, left=2cm, right=2cm, columnsep=0.75cm]{geometry}

\usepackage{xcolor}
\usepackage{times}
\usepackage{microtype}
\usepackage{soul}
\usepackage{url}
\usepackage[hidelinks]{hyperref}
\usepackage[utf8]{inputenc}
\usepackage[small]{caption}
\usepackage{graphicx}
\usepackage{amsmath}
\usepackage{mathtools}
\usepackage{cases}
\usepackage{amsthm}
\usepackage{amsfonts}
\usepackage{booktabs}
\usepackage{algorithm}
\usepackage{algorithmic}
\usepackage{textgreek}
\usepackage{csquotes}
\MakeOuterQuote{"}
\usepackage{newunicodechar}
\newunicodechar{μ}{\textmu}
\newunicodechar{ω}{\textomega}
\makeatletter

\def\normalsize{\@setfontsize\normalsize\@xpt{11}}
\def\small{\@setfontsize\small\@ixpt{10}}
\def\footnotesize{\@setfontsize\footnotesize\@ixpt{10}}
\def\scriptsize{\@setfontsize\scriptsize\@viipt{10}}
\def\tiny{\@setfontsize\tiny\@vipt{7}}
\def\large{\@setfontsize\large\@xipt{12}}
\def\Large{\@setfontsize\Large\@xiipt{14}}
\def\LARGE{\@setfontsize\LARGE\@xivpt{16}}
\def\huge{\@setfontsize\huge\@xviipt{20}}
\def\Huge{\@setfontsize\Huge\@xxpt{23}}

\newlength\titlebox 
\def\maketitle{\par
  \begingroup
  \def\thefootnote{\fnsymbol{footnote}}
  \def\@makefnmark{$^{\@thefnmark}$}
  \twocolumn[\@maketitle] \@thanks
  \endgroup
  \setcounter{footnote}{0}
  \let\maketitle\relax \let\@maketitle\relax
\gdef\@thanks{}\gdef\@author{}\gdef\@title{}\let\thanks\relax}

\def\@maketitle{%
  \newsavebox{\titlearea}
  \sbox{\titlearea}{%
    \let\footnote\thanks\relax
    \vbox{%
      \hsize\textwidth \linewidth\hsize
      \vskip 0.5in
      \centering
      {\LARGE\bf \@title \par}%
      \vskip 0.1in
      {%
        \def\and{\unskip\thinspace{\rm ,}\enspace}%
        \def\And{\unskip\enspace{\rm and}\enspace}%
        \def\affiliations{\egroup\par\Large\bgroup\rm}%
        \def\emails{\egroup\par\Large\bgroup\rm}%
        \bgroup\Large\bf\@author\egroup%
      }%
      \vskip 0.2in
    }%
  }%
  \newlength\actualheight
  \settoheight{\actualheight}{\usebox{\titlearea}}
  \ifdim\actualheight>\titlebox
  \setlength{\titlebox}{\actualheight}
  \fi
  \setcounter{footnote}{0}
  \vbox to \titlebox{%
    \def\thanks##1{\footnotemark}\relax
    \hsize\textwidth \linewidth\hsize
    \vskip 0.5in
    \centering
    {\LARGE\bf \@title \par}%
    \vskip 0.2in plus 4fil minus 0.1in
    {%
      \def\and{\unskip\thinspace{\rm ,}\enspace}%
      \def\And{\unskip\enspace{\rm and}\enspace}%
      \def\affiliations{\egroup\vskip 0.05in minus 0.05in\par\bgroup\Large\rm}%
      \def\emails{\egroup\vskip 0.05in minus 0.05in\par\bgroup\Large\rm}%
      \bgroup\Large\bf\@author\egroup%
    }%
    \vskip 0.3in plus 8fil minus 0.1in
  }%
}

\renewenvironment{abstract}{%
  \centerline{\bf Abstract}\vspace{0.5ex}
  \begin{quote}\small
  }{%
    \par
  \end{quote}\vskip 1ex
}

\def
\section{\@startsection{section}{1}{\z@}{-2.0ex plus -0.5ex minus -.2ex}{3pt plus 2pt minus 1pt}{\Large\bf\centering}}
\def
\subsection{\@startsection{subsection}{2}{\z@}{-2.0ex plus -0.5ex minus -.2ex}{3pt plus 2pt minus 1pt}{\large\bf\raggedright}}
\def
\subsubsection{\@startsection{subparagraph}{3}{\z@}{-6pt plus -2pt minus -1pt}{-1em}{\normalsize\bf}}
\skip\footins 9pt plus 4pt minus 2pt
\def\footnoterule{\kern-3pt \hrule width 5pc \kern 2.6pt}
\labelwidth\leftmargini\advance\labelwidth-\labelsep \labelsep 5pt
\def\@listi{\leftmargin\leftmargini}
\def\@listii{\leftmargin\leftmarginii
  \labelwidth\leftmarginii\advance\labelwidth-\labelsep
  \topsep 2pt plus 1pt minus 0.5pt
  \parsep 1pt plus 0.5pt minus 0.5pt
\itemsep \parsep}
\def\@listiii{\leftmargin\leftmarginiii
  \labelwidth\leftmarginiii\advance\labelwidth-\labelsep
  \topsep 1pt plus 0.5pt minus 0.5pt
  \parsep \z@ \partopsep 0.5pt plus 0pt minus 0.5pt
\itemsep \topsep}

\belowdisplayskip \abovedisplayskip
\def\leftcite{(}\def\rightcite{)}
\def\cite{\def\citeauthoryear##1##2{\def\@thisauthor{##1}%
\ifx\@lastauthor\@thisauthor\relax\else##1 \fi ##2}\@icite}
\def\shortcite{\def\citeauthoryear##1##2{##2}\@icite}
\def\citeauthor{\def\citeauthoryear##1##2{##1}\@nbcite}
\def\citeyear{\def\citeauthoryear##1##2{##2}\@nbcite}
\def\@icite{\leavevmode\def\@citeseppen{-1000}%
  \def\@cite##1##2{\leftcite\nobreak\hskip 0in{##1\if@tempswa , ##2\fi}\rightcite}%
  \@ifnextchar[{\@tempswatrue\@citex}{\@tempswafalse\@citex[]}}
  \def\@nbcite{\leavevmode\def\@citeseppen{1000}%
    \def\@cite##1##2{{##1\if@tempswa , ##2\fi}}%
    \@ifnextchar[{\@tempswatrue\@citex}{\@tempswafalse\@citex[]}}
    \def\@citex[#1]#2{%
      \def\@lastauthor{}\def\@citea{}%
      \@cite{\@for\@citeb:=#2\do
        {\@citea\def\@citea{;\penalty\@citeseppen\ }%
          \if@filesw\immediate\write\@auxout{\string\citation{\@citeb}}\fi
          \@ifundefined{b@\@citeb}{\def\@thisauthor{}{\bf ?}\@warning
          {Citation `\@citeb' on page \thepage \space undefined}}%
    {\csname b@\@citeb\endcsname}\let\@lastauthor\@thisauthor}}{#1}}

    \def\@lbibitem[#1]#2{
    \item\if@filesw
      {\def\protect##1{\string ##1\space}\immediate
    \write\@auxout{\string\bibcite{#2}{#1}}}\fi\ignorespaces}
    \def\@biblabel#1{}
    \def\thebibliography#1{
      \section*{References\@mkboth{REFERENCES}{REFERENCES}}\list
      {}{\labelwidth 0in\leftmargin\labelwidth
      \itemsep .01in}%
      \def\newblock{\hskip .11em plus .33em minus .07em}%
      \sloppy\clubpenalty4000\widowpenalty4000
    \sfcode`\.=1000\relax}

    \def\thm@space@setup{%
      \thm@preskip=0pt
      \thm@postskip=0pt
    }

    \makeatother

    \newtheorem{theorem}{Theorem}
    \newtheorem{definition}{Definition}
    \newtheorem{proposition}{Proposition}
    \newtheorem{lemma}{Lemma}
    \newtheorem{corollary}{Corollary}

    \title{On Halting vs Converging in Recurrent Graph Neural Networks}

    \author{%
      Jeroen Bollen\and
      Stijn Vansummeren\and
      \affiliations
      Hasselt University\\
      \emails
      jeroen.bollen@uhasselt.be, stijn.vansummeren@uhasselt.be
    }

    \newcommand{\concat}{\mathbin{|}}

    \newcommand{\Nat}{\mathbb{N}}
    \newcommand{\Real}{\mathbb{R}}
    \newcommand{\Bool}{\mathbb{B}}

    \newcommand{\card}[1]{\lvert #1 \rvert}
    \newcommand{\Mset}[1]{\mathcal{M}(#1)}
    \newcommand{\supp}{\mathrm{supp}}
    \newcommand{\mset}[1]{\left\{\!\!\{#1\}\!\!\right\}}

    \newcommand{\Graphs}[1]{\mathcal{G}[#1]}
    \newcommand{\nbr}[2]{\mathrm{nbr}_{#1}(#2)}

    \newcommand{\trafficlights}{\mathcal{T}}
    \newcommand{\adv}{\mathrm{adv}}
    \newcommand{\ret}{\mathrm{ret}}
    \newcommand{\enc}{\mathrm{enc}_3}
    \newcommand{\BEHIND}{\text{\textsc{behind}}}
    \newcommand{\ALIGNED}{\text{\textsc{aligned}}}
    \newcommand{\EAGER}{\text{\textsc{eager}}}

\begin{document}

    \pagestyle{plain}

    \maketitle

    \begin{abstract}
      Recurrent Graph Neural Networks (RGNNs) extend standard GNNs by iterating message-passing until some stopping condition is met. Various RGNN models have been proposed in the literature. In this paper, we study three such models: converging RGNNs, where all vertex representations must stabilise; output-converging RGNNs, where only the output classifications must stabilise; and halting RGNNs, where a per-vertex halting classifier determines when to stop. We establish expressiveness relationships between these models: over undirected graphs, converging RGNNs are equally expressive as graded-bisimulation-invariant halting RGNNs, while output-converging RGNNs are at least as expressive. Combined with prior results on halting RGNNs, this shows that, relative to the classifiers expressible in monadic second-order logic (MSO), converging RGNNs express exactly the graded modal μ-calculus (μGML), and output-converging RGNNs express at least μGML. These results hold even when restricting to ReLU networks with sum aggregation. The main technical challenge is simulating halting RGNNs by converging ones: without a global halting classifier, vertices may locally decide to halt at different times, causing desynchronisation. We develop a ``traffic-light'' protocol that enables vertices to coordinate despite this asynchrony. Our results answer an open question from Bollen et al.~\shortcite{Bollen-2025-Halting-Recurrent-GNNs} and show that the RGNN model of Pflueger et al.~\shortcite{Pflueger-2024-Recurrent-Graph-Neural} retains full μGML expressiveness even when convergence is guaranteed.
    \end{abstract}

    \section{Introduction}\label{sec:intro}

    Graph Neural Networks (GNNs) are a class of machine-learning models operating on graphs. The most common GNN architecture is the \emph{aggregate-combine} GNN (AC-GNN)~\cite{Gilmer-2017-Neural-message-passing,Hamilton-2020-Graph-Representation}. In each layer of an AC-GNN, every vertex first \emph{aggregates} the feature vectors of its neighbours into a single summary, and then \emph{combines} this summary with its own feature vector to produce an updated feature vector. This simple message-passing scheme underlies most practical AC-GNN variants, including Graph Convolutional Networks \cite{Kipf-2017-Semi-Supervised-Classification}, GraphSAGE \cite{Hamilton-2017-Inductive-Representation-Learning}, Graph Isomorphism Networks \cite{Xu-2019-How-Powerful-are}, and Principal Neighbourhood Aggregation \cite{Corso-2020-Principal-Neighbourhood-Aggregation}. Throughout this paper, when we refer to GNNs, we mean AC-GNNs.

    A natural question is: what can these networks compute? The \emph{expressiveness} of a GNN architecture refers to the class of vertex classifiers it can express. Understanding expressiveness reveals fundamental capabilities and limitations of an architecture, and provides a basis for comparing different architectures. A fundamental property of AC-GNNs is their invariance under graded bisimulation~\cite{Morris-2019-Weisfeiler-and-Leman,Xu-2019-How-Powerful-are}, which has enabled a productive line of research characterising GNN expressiveness through modal logics.

    For fixed-depth GNNs, this expressiveness is well understood. Barceló et al.~\shortcite{Barceló-2020-The-Logical-Expressiveness} showed that the classifiers expressible by both fixed-depth GNNs and first-order logic are precisely those definable in graded modal logic. For recurrent GNNs (RGNNs), which repeatedly apply a single layer to produce a sequence of increasingly refined feature vectors, the expressiveness is less understood, and is tightly linked to the question of how and especially when should the output be read.

    In this paper, we study three output semantics inspired by existing work and study their relative expressiveness: \emph{converging} RGNNs, where the entire feature vector of every vertex must stabilise~\cite{Scarselli-2009-The-Graph-Neural} (Definition~\ref{def:converging-rgnn}); \emph{output-converging} RGNNs, where only the classification must stabilise~\cite{Pflueger-2024-Recurrent-Graph-Neural} (Definition~\ref{def:output-converging-rgnn}); and \emph{halting} RGNNs, where an explicit halting classifier determines when to stop~\cite{Bollen-2025-Halting-Recurrent-GNNs} (Definition~\ref{def:halting-rgnn}). Our formalisations differ from these references in some respects; we give precise definitions in Section~\ref{sec:rgnn} and discuss the differences. Our main result is that, over undirected graphs, converging RGNNs have exactly the same expressive power as graded-bisimulation-invariant halting RGNNs (Theorem~\ref{thm:equivalence}). This resolves an open question posed by Bollen et al.~\shortcite{Bollen-2025-Halting-Recurrent-GNNs}, who asked whether it is possible to achieve fully convergent termination.

    The equivalence follows from two constructions that simulate one model with the other (Propositions~\ref{prop:converging-to-halting} and~\ref{prop:halting-to-converging}). The halting-to-converging direction proves especially challenging: in a halting RGNN, the computation stops only when all vertices satisfy the halting classifier simultaneously, providing a global synchronisation mechanism. A converging RGNN has no such mechanism, so vertices must coordinate among themselves when to stop. We solve this using a local coordination scheme detailed in Section~\ref{sec:halt-to-conv}. We also consider "simple" RGNNs, which use only standard neural network components; the converging-to-halting construction works even for simple RGNNs, while the halting-to-converging construction requires mild restrictions on the halting RGNN.

    Since every converging RGNN is trivially also an output-converging RGNN, our equivalence implies that output-converging RGNNs are at least as expressive as graded-bisimulation-invariant halting RGNNs. Combined with a recent result of Bollen et al.~\shortcite{Bollen-2025-Halting-Recurrent-GNNs}, we obtain that every vertex classifier definable in the graded modal μ-calculus (μGML) is expressible by a simple converging RGNN, and hence also by a simple output-converging RGNN. In the converse direction, converging and output-converging RGNNs are both invariant under graded bisimulation. Colcombet et al.~\shortcite{Colcombet-2025-Tree-Algebras-and} recently proved the finitary Janin-Walukiewicz theorem for standard bisimulation; if this extends to graded bisimulation, it follows that every MSO-definable classifier expressible by a converging or output-converging RGNN is also expressible in μGML. Ahvonen et al.~\shortcite{Ahvonen-2025-Graph-neural-networks} claim this characterisation in a preprint, proceeding via distributed message-passing automata; their work also relies on this extension, has not yet been published, and does not establish the characterisation for simple RGNNs. Our main contribution, the structural equivalence (Theorem~\ref{thm:equivalence}), is independent of this open question.

    Section~\ref{sec:prelim} introduces notation and AC-layers. Section~\ref{sec:rgnn} defines recurrent GNNs and their three output semantics. Section~\ref{sec:main} states the main results. Sections~\ref{sec:conv-to-halt} and~\ref{sec:halt-to-conv} prove the two directions of Theorem~\ref{thm:equivalence}. Section~\ref{sec:related} discusses related work, and Section~\ref{sec:discussion} discusses limitations and open questions.

    \section{Preliminaries}\label{sec:prelim}

    We denote the booleans, natural numbers including zero, and real numbers by \(\Bool\), \(\Nat\), and \(\Real\), respectively. For the booleans \(\Bool = \{0,1\}\), we interpret \(0\) as false and \(1\) as true. Given a set \(X\), we write \(\card{X}\) for its cardinality.

    A \emph{finite multiset} over a set \(X\) is a function \(M : X \to \Nat\) with support \(\supp(M) := \{x \in X \mid M(x) > 0\}\). Intuitively, \(M(x)\) denotes the multiplicity of \(x\) in \(M\). We use double curly braces \(\mset{...}\) to denote multisets.

    We denote the concatenation of two vectors \(x\) and \(y\) by \(x \concat y\). We write \([x]_i\) for the \(i\)-th component of a vector \(x\). The \(L_1\) norm of a vector \(x \in \Real^d\) is \(\|x\|_1 = \sum_{i=1}^d |[x]_i|\).

    \paragraph{Graphs.}
    A \emph{labelled graph} over a set of labels \(X\) is a triple \(G = (V, E, \lambda)\) where \(V\) is a finite set of vertices, \(E \subseteq V \times V\) is the edge relation, and \(\lambda : V \to X\) is the vertex labelling function. We say \(G\) is \emph{undirected} if \(E\) is symmetric.

    Throughout, we will always assume graphs are undirected. We write \(G(v)\) for the label of vertex \(v\) in \(G\), and define the neighbourhood of \(v\) as \(\nbr{G}{v} := \{u \in V \mid (v,u) \in E\}\). We write \(\Graphs{X}\) for the set of all finite undirected \(X\)-labelled graphs.

    \paragraph{Vertex Label Transformers.}

    Our primary objects of study are functions that transform the labels of a graph while preserving its structure.

    \begin{definition}[Label transformer]\label{def:label-transformer}
      A \emph{label transformer} from \(X\) to \(Y\) is a function \(f : \Graphs{X} \to \Graphs{Y}\) that, given an \(X\)-labelled graph \(G = (V, E, \lambda_1)\), produces a \(Y\)-labelled graph \(f(G) = (V, E, \lambda_2)\)
      . In other words, \(f\) preserves the vertex set and edge relation of the input graph, modifying only the vertex labels.
    \end{definition}

    Given a function \(h : X \to Y\), its \emph{lifting} \(h^\uparrow : \Graphs{X} \to \Graphs{Y}\) is the label transformer \(h^\uparrow(G) = (V, E, h \circ \lambda)\) that applies \(h\) to each vertex label of \(G\).

    \begin{definition}[Vertex classifier]\label{def:vertex-classifier}
      A \emph{vertex classifier} on \(X\) is a label transformer \(f : \Graphs{X} \to \Graphs{\Bool}\).
    \end{definition}

    Vertex classifiers are the central notion of expressiveness for GNNs: we say that a GNN \emph{expresses} a vertex classifier \(f\) if, on every input graph \(G\), the GNN produces the same Boolean labelling as \(f\).

    For a finite label set \(X\), write \(E\) for the edge relation and, for each \(q \in X\), write \(P_q\) for the set of vertices with label \(q\). A monadic second-order logic (MSO) formula \(\varphi(x)\) with one free vertex variable, or a μGML formula \(\varphi\), over these predicates defines a vertex classifier on \(X\), where vertex \(v\) is classified as true precisely when \(G \models \varphi(v)\). We say a vertex classifier is \emph{MSO-definable}, respectively \emph{definable in μGML}, if it is defined by such a formula.

    \paragraph{GNN Layers.}

    The basic computational unit of a graph neural network is the \emph{aggregate-combine layer}, which updates vertex labels by combining each vertex's current label with an aggregated summary of its neighbours' labels.

    \begin{definition}[Aggregate-combine layer]\label{def:ac-layer}
      An \emph{aggregate-combine layer} (AC-layer) of input dimension \(p\) and output dimension \(q\) is a label transformer \(L : \Graphs{\Real^p} \to \Graphs{\Real^q}\) of the form
      \begin{align}\label{eq:ac-layer}
        L(G)(v) = \mathrm{CMB}\Bigl(&G(v),\nonumber\\
        &\mathrm{AGG}\bigl(\mset{G(u) \mid u \in \nbr{G}{v}}\bigr)\Bigr),
      \end{align}
      where \(\mathrm{AGG} : \Mset{\Real^p} \to \Real^p\) is its \emph{aggregation function} and \(\mathrm{CMB} : \Real^p \times \Real^p \to \Real^q\) is its \emph{combination function}.
    \end{definition}

    While in their full generality, the \(\mathrm{AGG}\) and \(\mathrm{CMB}\) functions of an AC-layer can be arbitrary, in practice they are often instantiated to be of the following simple form.

    A function \(f : \Real^p \to \Real^q\) is called \emph{simple} if it is expressible as a feedforward neural network with ReLU activations, i.e., if \(f\) is of the form \(A_\ell \circ \mathrm{ReLU} \circ A_{\ell-1} \circ \cdots \circ \mathrm{ReLU} \circ A_1\), with each \(A_i\) an affine transformation and \([\mathrm{ReLU}(x)]_i = \max(0, [x]_i)\). A \emph{simple aggregation function} is the summation of its inputs, i.e., \(\mathrm{AGG}(M) = \sum_{x \in \supp(M)} M(x) \cdot x\). A \emph{simple combination function} is one of the form \(\mathrm{CMB}(x, y) = f(x \concat y)\) for some simple function \(f : \Real^{2p} \to \Real^q\). A \emph{simple AC-layer} is an AC-layer whose aggregation function is a simple aggregation function and whose combination function is a simple combination function.

    \section{Recurrent Graph Neural Networks}\label{sec:rgnn}

    A recurrent GNN iterates a single AC-layer, producing a sequence of feature graphs.

    \begin{definition}[Recurrent GNN]\label{def:rgnn}
      A \emph{recurrent graph neural network} (RGNN) over a label set \(X\) is a tuple \(\mathcal{R} = (\mathrm{In}, L, \mathrm{Out})\) where:
      \begin{itemize}
        \item \(\mathrm{In} : X \to \Real^d\) is the \emph{initialisation function};
        \item \(L : \Graphs{\Real^d} \to \Graphs{\Real^d}\) is an AC-layer, called the \emph{update function}; and
        \item \(\mathrm{Out} : \Real^d \to \Bool\) is the \emph{readout function}.
      \end{itemize}
      We call \(d\) the \emph{dimension} of \(\mathcal{R}\).
    \end{definition}

    \begin{definition}[Run]\label{def:run}
      The \emph{infinite run} of an RGNN \(\mathcal{R} = (\mathrm{In}, L, \mathrm{Out})\) on an input graph \(G \in \Graphs{X}\) is the sequence \(H_0, H_1, \ldots\) of \(\Real^d\)-labelled graphs where \(H_0 = \mathrm{In}^\uparrow(G)\) and \(H_{i+1} = L(H_i)\) for all \(i \geq 0\). A \emph{run} is any finite prefix \(H_0, \ldots, H_k\) of the infinite run. The \emph{output} of a run \(H_0, \ldots, H_k\) is the \(\Bool\)-labelled graph \(\mathrm{Out}^\uparrow(H_k)\).
    \end{definition}

    Different output semantics yield different architectures with different expressiveness. Throughout the following definitions, let \(\mathcal{R} = (\mathrm{In}, L, \mathrm{Out})\) be a fixed RGNN of dimension \(d\).

    \begin{definition}[Converging RGNN]\label{def:converging-rgnn}
      A run \(H_0, \ldots, H_k\) \emph{converges} if \(L(H_k) = H_k\). A \emph{converging RGNN} is an RGNN such that for every input graph, some run converges. The output on input \(G\) is
      \begin{align}
        \mathcal{C}(G) := \mathrm{Out}^\uparrow(H_k),
      \end{align}
      where \(H_0, \ldots, H_k\) is any converging run. Note that all converging runs on the same input yield the same output.
    \end{definition}

    \begin{definition}[Output-converging RGNN]\label{def:output-converging-rgnn}
      An infinite run \(H_0, H_1, \ldots\) \emph{output-converges} if there exists
      \(k \in \Nat\) such that \(\mathrm{Out}^\uparrow(H_i) = \mathrm{Out}^\uparrow(H_k)\) for all \(i \geq k\). An \emph{output-converging RGNN} is an RGNN such that for every input graph, the infinite run output-converges. The output on input \(G\) is
      \begin{align}
        \mathcal{O}(G) := \mathrm{Out}^\uparrow(H_k).
      \end{align}
    \end{definition}

    \begin{definition}[Halting RGNN]\label{def:halting-rgnn}
      A run \(H_0, \ldots, H_k\) is \emph{complete under halting classifier} \(\mathrm{Hlt} : \Real^d \to \Bool\) if \(k\) is the smallest index at which \(\mathrm{Hlt}(H_k(v)) = 1\) for every vertex \(v\). A \emph{halting RGNN} is a pair \(\mathcal{H} = (\mathcal{R}, \mathrm{Hlt})\) such that for every input graph, a complete run exists. The output on input \(G\) is
      \begin{align}
        \mathcal{H}(G) := \mathrm{Out}^\uparrow(H_k),
      \end{align}
      where \(H_0, \ldots, H_k\) is the complete run.
    \end{definition}

    A \emph{simple classifier} is a function \(c : \Real^d \to \Bool\) for which there exists a simple function \(f : \Real^d \to \Real\) such that \(c(x) = 1\) iff \(f(x) \geq 0\). We say an RGNN is \emph{simple} if its update function is a simple AC-layer and its readout function is a simple classifier. A converging RGNN is \emph{simple} if the underlying RGNN is simple. A halting RGNN is \emph{simple} if the underlying RGNN is simple and its halting classifier is a simple classifier.

    Converging RGNNs are inspired by the model of Scarselli et al.~\shortcite{Scarselli-2009-The-Graph-Neural}, who required the update function to be a \emph{contraction mapping}. By the Banach fixed-point theorem, iterating a contraction guarantees asymptotic convergence to a unique fixed point, but the fixed point need not be reached in finitely many steps. In practice, Scarselli et al.\ stop when the change between successive iterations falls below a threshold, which corresponds to a halting RGNN in our framework. Our converging RGNNs instead require exact stabilisation in finitely many steps, without imposing the contraction condition.

    Output-converging RGNNs are inspired by the model of Pflueger et al.~\shortcite{Pflueger-2024-Recurrent-Graph-Neural}, who define the output of a vertex to be the stabilised classification if one exists, and false otherwise. Their semantics assigns an output to every vertex on every graph, but their constructions do not guarantee that stabilisation occurs. We require that output-convergence is guaranteed for all inputs.

    Halting RGNNs correspond to the model introduced by Bollen et al.~\shortcite{Bollen-2025-Halting-Recurrent-GNNs}. Their notion of simplicity differs slightly from ours: they require the halting classifier to be a projection \(\mathrm{Hlt}(x) = 1\) iff \([x]_i > 0\) for some fixed index \(i\), whereas we allow any simple function with a \(\geq 0\) threshold. Since their constructions produce feature vectors where the halting and readout components are boolean, the strict inequality \([x]_i > 0\) is equivalent to \([x]_i - \tfrac{1}{2} \geq 0\), which fits our definition.

    \paragraph{Graded Bisimulation.}

    Graded bisimulation is a notion of structural equivalence for graphs that respects neighbourhood multiplicities. It characterises the distinguishing power of GNNs: two pointed graphs are indistinguishable by any GNN if and only if they are graded bisimilar (Grohe~\shortcite{Grohe-2021-The-Logic-of}).

    \begin{definition}[Graded bisimulation]\label{def:g-bisim}
      Let \(G\) and \(H\) be \(X\)-labelled graphs. A relation \(Z \subseteq V_G \times V_H\) is a \emph{graded bisimulation} between \(G\) and \(H\) if for every \((u, v) \in Z\):
      \begin{enumerate}
        \item \(G(u) = H(v)\); and
        \item there exists a bijection \(f : \nbr{G}{u} \to \nbr{H}{v}\) such that \((w, f(w)) \in Z\) for every \(w \in \nbr{G}{u}\).
      \end{enumerate}
      A graded bisimulation is \emph{total} if its domain is the entire vertex set of the first graph \(G\), and \emph{surjective} if its range is the entire vertex set of the second graph \(H\).
    \end{definition}

    The bijection requirement in condition (2) ensures that neighbours are matched one-to-one, preserving multiplicities. In ordinary bisimulation, multiple neighbours can be matched to the same neighbour, ignoring multiplicities.

    \begin{definition}[Invariance under graded-bisimulation]\label{def:g-bisim-invariance}
      A label transformer \(f : \Graphs{X} \to \Graphs{Y}\) is \emph{invariant under graded-bisimulation} if for every pair of \(X\)-labelled graphs \(G\) and \(H\) and every graded-bisimulation \(Z\) between \(G\) and \(H\), the relation \(Z\) is also a graded-bisimulation between \(f(G)\) and \(f(H)\).
    \end{definition}

    \begin{proposition}\label{prop:rgnn-invariance}
      Every vertex classifier definable by a converging RGNN or an output-converging RGNN is invariant under graded bisimulation. Every vertex classifier definable by a halting RGNN is invariant under total surjective graded bisimulations.
    \end{proposition}

    \begin{proof}
      Every AC-layer preserves graded-bisimulations, as does every lifted function, and any composition thereof. As runs of RGNNs are precisely such compositions, a graded-bisimulation between input graphs remains one after any number of iterations.

      For converging RGNNs, this implies that graded-bisimilar vertices converge at the same iteration with equal labels: if one has stabilised, so has the other.

      The output-converging case was proven by Pflueger et al.~\shortcite{Pflueger-2024-Recurrent-Graph-Neural}.

      For halting RGNNs, the situation is more delicate. Halting depends on a global condition: all vertices must satisfy the halting classifier. Totality and surjectivity ensure that if all vertices halt in one graph, so do all vertices in the other. The full proof is given by Bollen et al.~\shortcite{Bollen-2025-Halting-Recurrent-GNNs}.
    \end{proof}

    Clearly, invariance under graded bisimulation implies invariance under
    total-surjective graded bisimulation, but the converse is not
    true. Consequently, any vertex classifier that is definable by a halting RGNN
    and that is only invariant under total surjective graded bisimulation, but not
    under  graded bisimulation, cannot equivalently be defined by a
    (output-) converging RGNN.  Therefore, in the following we necessarily compare
    converging RGNNs with halting RGNNs on the class of vertex classifiers that are
    invariant under graded bisimulation.

    \section{Main Results}\label{sec:main}

    Our main contribution is establishing the equivalence between converging RGNNs and graded-bisimulation-invariant halting RGNNs.

    \begin{theorem}\label{thm:equivalence}
      Over undirected graphs, converging RGNNs and graded-bisimulation-invariant halting RGNNs express exactly the same vertex classifiers.
    \end{theorem}

    This theorem follows from two constructions, detailed in Sections~\ref{sec:conv-to-halt} and~\ref{sec:halt-to-conv}:

    \begin{proposition}\label{prop:converging-to-halting}
      For every converging RGNN \(\mathcal{C}\), there exists a halting RGNN \(\mathcal{H}\) such that \(\mathcal{C}(G) = \mathcal{H}(G)\) for all graphs \(G\). If \(\mathcal{C}\) is simple, then so is \(\mathcal{H}\).
    \end{proposition}

    \begin{proposition}\label{prop:halting-to-converging}
      For every graded-bisimulation-invariant halting RGNN \(\mathcal{H}\), there exists a converging RGNN \(\mathcal{C}\) such that \(\mathcal{H}(G) = \mathcal{C}(G)\) for all undirected graphs \(G\). Furthermore, if $\mathcal{H}$ satisfies the following conditions, then $\mathcal{C}$ is simple:
      \begin{enumerate}
        \item \(\mathcal{H}\) is simple;
        \item the simple function \(f\) defining \(\mathcal{H}\)'s halting classifier satisfies \(f(H_i(v)) \in \{-1, 1\}\) for all vertices \(v\) and iterations \(i\) of every run;
          and
        \item there exists \(B \geq 0\) such that
          \(|[H_{i+1}(v)]_j - [H_i(v)]_j| \leq B\) for all feature vector components
          \(j\), vertices \(v\), and iterations \(i\) of every run.
      \end{enumerate}
      We refer to the second condition as the \emph{range condition} and the third condition as the \emph{bounded-change condition}.
    \end{proposition}

    Bollen et al.~\shortcite{Bollen-2025-Halting-Recurrent-GNNs} showed that every μGML formula is expressible by a graded-bisimulation-invariant simple halting RGNN whose halting classifier satisfies the range condition of Proposition~\ref{prop:halting-to-converging}. Their constructions have per-component changes in \(\{-1, 0, 1\}\), so the bounded-change condition is satisfied with \(B = 1\). Combined with Proposition~\ref{prop:halting-to-converging}, we obtain:

    \begin{corollary}\label{cor:mu-gml}
      Every vertex classifier definable in μGML is expressible by a simple converging RGNN, and hence also by a simple output-converging RGNN.
    \end{corollary}

    For the converse direction, converging RGNNs and output-converging RGNNs are both invariant under graded bisimulation as per Proposition~\ref{prop:rgnn-invariance}. Colcombet et al.~\shortcite{Colcombet-2025-Tree-Algebras-and} recently proved the finitary Janin-Walukiewicz theorem~\cite{Janin-2001-Relating-Levels-of,Walukiewicz-2002-Monadic-second-order}: over finite graphs, bisimulation-invariant MSO formulas are expressible in the μ-calculus. Whether this result extends to graded bisimulation is, to our knowledge, open; if it does, every MSO-definable vertex classifier expressible by a converging or output-converging RGNN is also expressible in μGML.

    \section{From Converging to Halting}\label{sec:conv-to-halt}

    We begin with the simpler direction: every converging RGNN can be simulated by a halting RGNN.

    \begin{proof}[Proof of Proposition~\ref{prop:converging-to-halting}]
      Let \(\mathcal{C} = (\mathrm{In}, L, \mathrm{Out})\) be a converging RGNN of dimension \(d\), and let \(H_0, H_1, \ldots\) denote the infinite run of \(\mathcal{C}\) on an input graph \(G\). Since \(\mathcal{C}\) converges, there exists some \(k \in \Nat\) such that \(H_{k-1} = H_k\).

      The idea is to construct a halting RGNN \(\mathcal{H}\) whose run \(H'_0, H'_1, \ldots\) satisfies \(H'_i(v) = H_i(v) \concat H_{i-1}(v)\) for all \(i \geq 1\). By storing both the current and previous feature vectors, each vertex can locally detect whether its feature has stabilised. When \(H_{k-1} = H_k\), all vertices will simultaneously detect stabilisation, triggering the system to halt.

      We construct \(\mathcal{H} = ((\mathrm{In}', L', \mathrm{Out}'), \mathrm{Hlt})\) of dimension \(2d\) as follows:
      \begin{itemize}
        \item \(\mathrm{In}'(x) := \mathrm{In}(x) \concat (\mathrm{In}(x) + \mathbf{1}_d)\), where \(\mathbf{1}_d \in \Real^d\) is the all-ones vector. This ensures the two components are initially distinct, so no vertex signals halting before the first iteration.
        \item \(L'\) is an AC-layer that, given a feature vector \(y \concat y' \in \Real^{2d}\), computes \(L(y) \concat y\).
        \item \(\mathrm{Hlt}(y \concat y') := 1\) iff \(h(y \concat y') \geq 0\), where \(h(y \concat y') := -\|y - y'\|_1\).
        \item \(\mathrm{Out}'(y \concat y') := \mathrm{Out}(y)\).
      \end{itemize}

      We verify that \(L'\) is indeed an AC-layer. Since \(L\) has aggregation function \(\mathrm{AGG}\) and combination function \(\mathrm{CMB}\), we define:
      \begin{align}
        \mathrm{AGG}'(M) &:= \mathrm{AGG}(\pi_1(M)) \concat \mathbf{0}_d, \\
        \mathrm{CMB}'((y \concat y'), (a \concat a')) &:= \mathrm{CMB}(y, a) \concat y,
      \end{align}
      where \(\pi_1(M)\) denotes the multiset obtained by projecting each element of \(M\) onto its first \(d\) components.

      By construction, the run \(H'_0, H'_1, \ldots\) of \(\mathcal{H}\) on \(G\) satisfies \(H'_i(v) = H_i(v) \concat H_{i-1}(v)\) for all \(i \geq 1\) and all vertices \(v\). At iteration \(k\), every vertex \(v\) has \(H'_k(v) = H_k(v) \concat H_{k-1}(v)\), and since \(H_{k-1} = H_k\), we have \(h(H'_k(v)) = -\|H_k(v) - H_{k-1}(v)\|_1 = 0 \geq 0\), so all vertices satisfy the halting condition. Thus \(\mathcal{H}\) halts, and the outputs agree:
      \begin{align}
        \mathcal{H}(G) = \mathrm{Out}'^\uparrow(H'_k) = \mathrm{Out}^\uparrow(H_k) = \mathcal{C}(G).
      \end{align}

      Now suppose \(\mathcal{C}\) is simple. The construction above does not preserve simplicity, since \(\mathrm{AGG}'\) involves projection rather than plain summation. We give an alternative construction of \(L'\) that preserves simplicity. Since \(\mathrm{AGG}\) is summation over \(\Real^d\), we define \(\mathrm{AGG}'\) to be summation over \(\Real^{2d}\). Since summation distributes over concatenation, the first \(d\) components of the result equal \(\mathrm{AGG}(\pi_1(M))\); the remaining components are ignored by \(\mathrm{CMB}'\). We define
      \begin{align}
        \mathrm{CMB}'((y \concat y'), (a \concat a')) := f(y \concat a) \concat y,
      \end{align}
      where \(f\) is the simple function such that \(\mathrm{CMB}(y, a) = f(y \concat a)\). This layer computes the same function as the general construction. Now \(\mathrm{AGG}'\) is summation, hence simple, and \(\mathrm{CMB}'\) is simple since it selects \(y\) and \(a\) from the input via an affine transformation, applies \(f\), and concatenates \(y\). The function \(h(y \concat y') = -\|y - y'\|_1 = -\sum_{i=1}^d |[y]_i - [y']_i|\) is simple, since \(|z| = \mathrm{ReLU}(z) + \mathrm{ReLU}(-z)\). If \(\mathrm{Out}\) is a simple classifier, then so is \(\mathrm{Out}'\).
    \end{proof}

    \section{From Halting to Converging}\label{sec:halt-to-conv}

    The converse direction is more challenging. Given a graded-bisimulation-invariant halting RGNN \(\mathcal{H}\), we must construct a converging RGNN \(\mathcal{C}\) that computes the same classifier.

    Throughout this section, let \(\mathcal{H}\) be a halting RGNN with complete run \(H_0, \ldots, H_k\) on input graph \(G\), where \(k\) is the index at which the run completes, and let \(C_0, C_1, \ldots\) be the infinite run of a converging RGNN \(\mathcal{C}\) on the same input. We refer to each graph \(H_i\) in the halting run as a \emph{halting state}, and each graph \(C_j\) in the converging run as a \emph{converging state}. A vertex's feature vector in a halting state, \(H_i(v)\), is called a \emph{snapshot}, and a vertex's feature vector in a converging state, \(C_j(v)\), is called a \emph{configuration}.

    \subsection{The Synchronisation Problem}\label{sec:sync-problem}

    The difficulty lies in coordination. In a halting RGNN, vertices can continue computing indefinitely until all vertices satisfy the halting classifier, at which point the entire system halts all at once. In a converging RGNN, vertices cannot simply keep computing while waiting for the system to halt: the act of continuing the computation prevents the feature vectors from stabilising, which is precisely the condition required for convergence.

    A natural strategy is to have each vertex step through \(\mathcal{H}\)'s halting states one at a time. A vertex pauses at \(H_i\) once it locally satisfies the halting classifier, and resumes when it discovers that a neighbour is already simulating \(H_{i+1}\). This causes vertices to desynchronise: different vertices may be simulating different halting states at the same time. The difficulty is that to advance, a vertex needs its neighbours' snapshots from the halting state it is currently simulating, but a neighbour that has already advanced no longer holds that snapshot.

    We solve this by having each vertex advertise its previous snapshot when it advances, so that neighbours who stayed behind can catch up. However, neighbours that advanced simultaneously receive this previous snapshot too, even though they no longer need it. Each vertex therefore signals both which halting state it is currently simulating and which halting state its advertised snapshot represents. This enables two things:
    \begin{itemize}
      \item A vertex can detect when a neighbour has moved ahead, triggering it to catch up.
      \item A vertex can determine whether a neighbour's advertised snapshot is the one it needs or an outdated one.
    \end{itemize}
    To make this precise, we first introduce a formal notion of which halting state each vertex is simulating.

    \subsection{Correspondence}\label{sec:correspondence}

    To relate the converging run to the halting run, we introduce a \emph{correspondence}: a function that maps, for each converging state, each vertex to a halting state.

    \begin{definition}[Correspondence]\label{def:correspondence}
      A \emph{correspondence} between these runs is a function \(\Phi : V \times \Nat \to \{0, \ldots, k\}\) satisfying, for all vertices \(v\) and all \(j \geq 0\):
      \begin{align}
        \Phi(v, j+1) \in \{\Phi(v, j), \Phi(v, j) + 1\}.
      \end{align}
      Given a converging state \(C_j\), we say that vertex \(v\) is simulating halting state \(H_{\Phi(v,j)}\).
    \end{definition}

    The monotonicity condition ensures that vertices only move forward through the halting run, advancing by at most one halting state per converging step.

    Given a correspondence \(\Phi\), we say vertex \(v\) is \emph{behind} at converging state \(C_j\) if there exists a neighbour \(u\) with \(\Phi(u,j) > \Phi(v,j)\). A vertex is \emph{eager} at \(C_j\) if it is behind or if \(\mathrm{Hlt}(H_{\Phi(v,j)}(v)) = 0\). A vertex is \emph{aligned} at \(C_j\) when it receives from each neighbour the information about \(H_{\Phi(v,j)}\) needed to advance. We will construct \(\mathcal{C}\) so that a vertex advances precisely when it is both eager and aligned. Lemma~\ref{lem:predicates} will show that these properties can be tested locally.

    \subsection{Configurations}\label{sec:config}

    Assume that \(\mathcal{H}\) has dimension \(d\). Each configuration \(C_j(v) \in \Real^{2d+6}\) has the form:
    \begin{align}
      C_j(v) &= \left(\kappa^c_j(v),\, \tau^c_j(v),\, \kappa^m_j(v),\, \tau^m_j(v)\right), \label{eq:config}
    \end{align}
    where \(\kappa^c_j(v), \kappa^m_j(v) \in \Real^d\) are the \emph{current} and \emph{advertised snapshot} respectively, and \(\tau^c_j(v), \tau^m_j(v) \in \Real^3\) are the \emph{current} and \emph{advertised traffic light}, defined next.

    The advertised components store information transmitted to neighbours; the traffic lights, defined next, encode which halting state a vertex is simulating, allowing neighbours to determine their relative positions in the halting run.

    \begin{definition}[Traffic light]
      A \emph{traffic light} is an element of
      \begin{align}
        \trafficlights := \{(1,0,0), (0,1,0), (0,0,1)\}.
      \end{align}
      The encoding \(\enc : \Nat \to \trafficlights\) maps \(i\) to \((1,0,0)\) if \(i \equiv 0 \pmod 3\), to \((0,1,0)\) if \(i \equiv 1 \pmod 3\), and to \((0,0,1)\) if \(i \equiv 2 \pmod 3\). The successor and predecessor operations are \(\adv(g, y, r) := (r, g, y)\) and \(\ret(g, y, r) := (y, r, g)\), so that \(\enc(i+1) = \adv(\enc(i))\) and \(\enc(i-1) = \ret(\enc(i))\) for \(i \geq 1\).
    \end{definition}

    \subsection{Coherence}\label{sec:coherence}

    A correspondence is \emph{coherent} when the configurations genuinely reflect the halting states being simulated.

    \begin{definition}[Coherence]\label{def:coherence}
      A correspondence \(\Phi\) is \emph{coherent} if for every converging state \(C_j\) and every vertex \(v\):
      \begin{enumerate}
        \item\label{coh:1} \(\kappa^c_j(v) = H_{\Phi(v,j)}(v)\);
        \item\label{coh:2} \(\tau^c_j(v) = \enc(\Phi(v, j))\);
        \item\label{coh:3} if \(j \geq 1\), then \(\kappa^m_j(v) = \kappa^c_{j-1}(v)\) and \(\tau^m_j(v) = \tau^c_{j-1}(v)\); and
        \item\label{coh:4} for every neighbour \(u\) of \(v\), \(|\Phi(v, j) - \Phi(u, j)| \leq 1\).
      \end{enumerate}
    \end{definition}

    We say \(\Phi\) is \emph{coherent until step \(n\)} if Coherence~\ref{coh:1}--\ref{coh:4} hold for all \(j \leq n\).

    Coherence items~\ref{coh:1} and~\ref{coh:2} require that each vertex's current snapshot and traffic light faithfully reflect the halting state being simulated. Coherence~\ref{coh:3} requires the advertised components to equal the previous current components, and Coherence~\ref{coh:4} bounds the desynchronisation between neighbours to at most one halting state.

    The choice of encoding \(\Phi(v,j)\) as a traffic light may seem odd, but storing it as a natural number would not suffice due to the aggregation mechanism of GNNs. In a simple GNN, each vertex receives only the sum of its neighbours' messages, not individual values. If vertices sent their \(\Phi\)-values directly, the sum would reveal only the total, making it impossible to determine whether any particular neighbour is behind. The same holds for most aggregation functions, such as mean and max.

    \subsection{Construction}\label{sec:construction}

    We construct \(\mathcal{C} = (\mathrm{In}', L', \mathrm{Out}')\) as follows. The initialisation function sets both snapshots to the initial halting state and both traffic lights to \(\enc(0)\):
    \begin{align}
      \mathrm{In}'(x) &\coloneqq \left(\mathrm{In}(x),\, \enc(0),\, \mathrm{In}(x),\, \enc(0)\right).
    \end{align}
    The update function \(L'\) is an AC-layer with aggregation function \(\mathrm{AGG}'\) and combination function \(\mathrm{CMB}'\). The aggregation function \(\mathrm{AGG}'\) applies \(\mathcal{H}\)'s aggregation to each snapshot component independently and summation to each traffic light component. If \(\mathcal{H}\) is simple, this reduces to summation over the full configuration.

    Since Coherence~\ref{coh:4} guarantees that neighbours differ by at most one halting state, comparing traffic lights determines their relative positions. Writing \(\odot\) for componentwise product, we define the following predicates on two configurations \(\kappa = (\kappa^c, \tau^c, \kappa^m, \tau^m)\) and  \(\hat{\kappa} = (\hat{\kappa^c}, \hat{\tau^c}, \hat{\kappa^m}, \hat{\tau^m})\):

    \begin{align}
      \label{eq:behind}
      \BEHIND(\kappa, \hat{\kappa}) &\coloneqq \|\adv(\tau^c) \odot \hat{\tau^c}\|_1 > 0, \\
      \label{eq:aligned}
      \ALIGNED(\kappa, \hat{\kappa}) &\coloneqq \|\tau^c \odot \hat{\tau^m}\|_1 = \|\hat{\tau^m}\|_1, \\
      \EAGER(\kappa, \hat{\kappa}) &\coloneqq \BEHIND(\kappa, \hat{\kappa}) \lor \neg\mathrm{Hlt}(\kappa^c).
    \end{align}
    When \(\kappa\) is the configuration of a vertex in some state \(C_j\) and \(\hat{\kappa}\) is the aggregated configurations from its neighbours, these predicates detect whether a vertex is behind, aligned, and eager, using only the traffic lights and the halting classifier. In particular, every vertex that is behind is eager.

    The combination function \(\mathrm{CMB}'\) distinguishes two cases: a vertex is \emph{advancing} if \(\ALIGNED(\kappa, \hat{\kappa}) \land \EAGER(\kappa, \hat{\kappa})\), and \emph{waiting} otherwise.

    Writing \(\kappa' \coloneqq \mathrm{CMB}(\kappa^c, \hat{\kappa^m})\) for the next snapshot and \(\tau' \coloneqq \adv(\tau^c)\) for the next traffic light:
    \begin{subnumcases}{\label{eq:comb-cases}\mathrm{CMB}'(\kappa, \hat{\kappa}) \coloneqq}
      (\kappa',\, \tau',\, \kappa^c,\, \tau^c) & \text{if advancing,} \label{eq:advancing} \\
      (\kappa^c,\, \tau^c,\, \kappa^c,\, \tau^c) & \text{if waiting.} \label{eq:waiting}
    \end{subnumcases}
    Note that both cases set \(\tau^m\) to the current traffic light \(\tau^c\).

    Finally, the readout function is:
    \begin{align}
      \mathrm{Out}'(\kappa) &\coloneqq \mathrm{Out}(\kappa^c). \label{eq:out}
    \end{align}
    If the correspondence is coherent and \(\Phi(v,j) = k\), then \(\kappa^c = H_k(v)\), so \(\mathrm{Out}'\) always returns the same value as \(\mathrm{Out}\) at the corresponding halting step: \(\mathrm{Out}'(\mathcal{C}_j(v)) = \mathrm{Out}(H_k(v))\). In particular, once \(\Phi(v,j) = k\) for all \(v\), this holds simultaneously.

    In the following lemma we argue formally that \(\mathcal{C}\) is simple if \(\mathcal{H}\) satisfies the three conditions of Proposition~\ref{prop:halting-to-converging}. Intuitively, the range condition is needed because \(\mathcal{C}\) applies Boolean-valued logic to the result of \(\mathcal{H}\)'s halting classifier \(\mathrm{Hlt}\). If the simple function defining \(\mathrm{Hlt}\) is arbitrary, then we cannot do this required boolean logic using simple functions only. By contrast, if it always produces values in \(\{-1, 1\}\), then these can easily be transformed to the range \(\{0, 1\}\), which allows us to do boolean logic using simple functions.

    The bounded-change condition is needed because conditionally applying a simple function, i.e., multiplying a branch by a \(\{0, 1\}\) indicator of the same input, is not itself simple. Instead, we implement the case distinction using a ReLU identity that requires the per-component changes to be bounded.

    \begin{lemma}\label{lem:simplicity}
      If \(\mathcal{H}\) is simple, the simple function \(f\) defining \(\mathrm{Hlt}\) has range \(\{-1, 1\}\) on all feature vectors encountered during runs of \(\mathcal{H}\), and there exists \(B \geq 0\) such that \(|[H_{i+1}(v)]_j - [H_i(v)]_j| \leq B\) for all components \(j\), vertices \(v\), and iterations \(i\) of every run, then \(\mathcal{C}\) is also simple.
    \end{lemma}
    \begin{proof}
      Since \(\mathcal{H}\) is simple, \(\mathrm{AGG}'\) is summation and hence simple. Writing \(|z| \coloneqq \mathrm{ReLU}(z) + \mathrm{ReLU}(-z)\), the norm \(\|x\|_1 = \sum_i |[x]_i|\) is simple. During runs, the norms in \(\BEHIND\) and \(\ALIGNED\) evaluate to non-negative integers, and \(\min(x, 1) = x - \mathrm{ReLU}(x - 1)\) maps them to \(\{0, 1\}\) indicators, returning $0$ if the comparisons in \eqref{eq:behind} and \eqref{eq:aligned} return False, and $1$ otherwise. Since the simple function \(f\) defining $\mathrm{Hlt}$ has range \(\{-1, 1\}\), the simple function \((1 - f)/2\) gives a \(\{0, 1\}\) indicator for \(\neg\mathrm{Hlt}\), so \(\EAGER\) is also simple.

      It remains to show that \(\mathrm{CMB}'\) is simple. Let \(a\) and \(e\) denote the \(\{0, 1\}\)-valued \ALIGNED\ and \EAGER\ predicates above, and let \(s \coloneqq \mathrm{ReLU}(a + e - 1)\), so that \(s = 1\) when advancing and \(s = 0\) when waiting.

      Both branches of~\eqref{eq:comb-cases} set the advertised components to \((\kappa^c, \tau^c)\), so these require no conditional logic. For the current-snapshot components, define the per-component delta \(\delta_i \coloneqq [\mathrm{CMB}(\kappa^c, \hat{\kappa^m})]_i - [\kappa^c]_i\) and the function
      \begin{align}
        \varphi(\delta, s') \coloneqq{} & \mathrm{ReLU}(\delta + s') - \mathrm{ReLU}(\delta) \nonumber \\
        & {} - \mathrm{ReLU}(-\delta + s') + \mathrm{ReLU}(-\delta). \label{eq:relu-cond}
      \end{align}
      It is straightforward to verify that \(\varphi(\delta, 0) = 0\) and \(\varphi(\delta, s') = \delta\) whenever \(s' \geq |\delta|\). The bounded-change condition ensures \(|\delta_i| \leq B\) on coherent inputs, so setting \(s' \coloneqq B \cdot s\) makes \([\kappa^c]_i + \varphi(\delta_i,\, B \cdot s)\) equal \([\kappa^c]_i\) when waiting (\(s = 0\)) and \([\mathrm{CMB}(\kappa^c, \hat{\kappa^m})]_i\) when advancing (\(s = 1\)). The traffic-light components are handled identically with \(B = 1\). All operations are compositions of affine transformations and \(\mathrm{ReLU}\), so \(\mathrm{CMB}'\) is simple; note it agrees with~\eqref{eq:comb-cases} only on coherent inputs. Finally, \(\mathrm{Out}'\) projects onto \(\kappa^c\) and applies \(\mathrm{Out}\), so it is also simple.
    \end{proof}

    \subsection{Correctness}\label{sec:correctness}

    It remains to show that \(\mathcal{C}\) is a converging RGNN computing the same classifier as \(\mathcal{H}\). To this end, throughout the section, fix input graph $G$, the complete run \(H_0, \ldots, H_k\) of $\mathcal{H}$ on  \(G\), and the infinite run \(C_0, C_1, \ldots\)  of \(\mathcal{C}\) on  $G$.

    The correctness argument relies on four key lemmas, which we first state before turning to their proof. The first lemma is Coherence, which provides an invariance relating \(\mathcal{C}\)'s infinite run to \(\mathcal{H}\)'s halting run.

    \begin{lemma}[Coherence]\label{lem:coherence}
      There exists a coherent correspondence \(\Phi\) between the halting run of \(\mathcal{H}\) and the converging run of \(\mathcal{C}\).
    \end{lemma}

    For the other three lemmas, fix a connected component \(\Gamma\) of \(G\) with vertex set \(N_\Gamma\), and let \(k_\Gamma\) be the first index at which \(\mathrm{Hlt}(H_{k_\Gamma}(v)) = 1\) for all \(v \in N_\Gamma\). By definition of a complete run, \(k_\Gamma \leq k\). Note \(k_\Gamma < k\) is possible when all vertices with \(\mathrm{Hlt}(H_{k_\Gamma}(v)) = 0\) lie outside \(N_\Gamma\). The remaining three key lemmas, called Completeness, Convergence, and Correctness, intuitively show that \(\mathcal{C}\)'s configurations on \(N_\Gamma\) converge to their state at \(H_{k_\Gamma}\), and that for all \(v \in N_\Gamma\), \(\mathrm{Out}(H_{k_\Gamma}(v)) = \mathrm{Out}(H_k(v))\) . To obtain the latter equality invariance of $\mathcal{H}$ under  graded-bisimulation is crucial: it implies that the computation of $\mathcal{H}$ at any vertex depends only on its connected component, so vertices from different components cannot influence each other.

    In the remaining lemmas, let \(\Phi\) denote the coherent correspondence given by Lemma~\ref{lem:coherence}.

    \begin{lemma}[Completeness]\label{lem:completeness}
      There exists a \(j' \in \Nat\) such that \(\Phi(v, j') = k_\Gamma\) for all \(v \in N_\Gamma\).
    \end{lemma}

    \begin{lemma}[Convergence]\label{lem:convergence}
      For all \(j \geq j' + 1\) and all \(v \in N_\Gamma\), \(\mathcal{C}_{j+1}(v) = \mathcal{C}_j(v)\).
    \end{lemma}

    \begin{lemma}[Correctness]\label{lem:output}
      For all \(j \geq j'\) and all \(v \in N_\Gamma\), \(\mathrm{Out}'(\mathcal{C}_j(v)) = \mathcal{H}(G)(v)\).
    \end{lemma}

    We next prove Lemmas~\ref{lem:coherence}--\ref{lem:output}.  We begin with a lemma confirming that, given a coherent correspondence the predicates \(\BEHIND\), \(\ALIGNED\), and \(\EAGER\) coincide with the definitions of \emph{behind}, \emph{aligned}, and \emph{eager} from Section~\ref{sec:correspondence}.

    \begin{lemma}\label{lem:predicates}
      Write \(\BEHIND_j(v)\), \(\ALIGNED_j(v)\), and \(\EAGER_j(v)\) for the predicates \(\BEHIND\), \(\ALIGNED\), and \(\EAGER\) evaluated at vertex \(v\) in state \(C_j\). Let \(\Phi\) be a correspondence coherent until step \(j\). Then:
      \begin{enumerate}
        \item\label{pred:behind} \(\BEHIND_j(v)\) holds if and only if some neighbour \(u\) of \(v\) satisfies \(\Phi(u,j) > \Phi(v,j)\). In that case, \(\Phi(u,j) = \Phi(v,j) + 1\).
        \item\label{pred:aligned} \(\ALIGNED_j(v)\) holds if and only if \(\kappa^m_j(u) = H_{\Phi(v,j)}(u)\) and \(\tau^m_j(u) = \enc(\Phi(v,j))\) for every neighbour \(u\) of \(v\).
        \item\label{pred:eager} \(\EAGER_j(v)\) holds if and only if \(v\) is behind or \(\mathrm{Hlt}(H_{\Phi(v,j)}(v)) = 0\).
      \end{enumerate}
    \end{lemma}
    \begin{proof}
      We prove the three statements in order.
      \begin{enumerate}
        \item By Coherence~\ref{coh:4}, every neighbour \(u\) has \(\Phi(u,j) \in \{\Phi(v,j) - 1,\, \Phi(v,j),\, \Phi(v,j) + 1\}\). By Coherence~\ref{coh:2}, \(\tau^c_j(u) = \enc(\Phi(u,j))\). Since \(\enc(i+1) = \adv(\enc(i))\) and \(\enc(i-1) = \ret(\enc(i))\), the three possible traffic lights are \(\ret(\tau^c_j(v))\), \(\tau^c_j(v)\), and \(\adv(\tau^c_j(v))\). These are three distinct one-hot vectors, so each is uniquely identifiable in the componentwise sum \(\hat{\tau^c_j}(v)\), and \(\|\adv(\tau^c_j(v)) \odot \hat{\tau^c_j}(v)\|_1 > 0\) precisely when some neighbour has \(\Phi(u,j) = \Phi(v,j) + 1\).

        \item By Coherence~\ref{coh:2}, \(\tau^c_j(v) = \enc(\Phi(v,j))\). Since traffic lights are one-hot, \(\ALIGNED_j(v)\) holds precisely when \(\tau^m_j(u) = \enc(\Phi(v,j))\) for every neighbour \(u\). At \(j = 0\), the definition of \(\mathrm{In}'\) gives \(\kappa^m_0(u) = H_0(u)\) and \(\tau^m_0(u) = \enc(0)\), so the claim is immediate. For \(j \geq 1\), by Coherence~\ref{coh:3}, \(\kappa^m_j(u) = \kappa^c_{j-1}(u)\) and \(\tau^m_j(u) = \tau^c_{j-1}(u)\), so by Coherence~\ref{coh:1} and~\ref{coh:2}, \(\kappa^m_j(u) = H_{\Phi(u,j-1)}(u)\) and \(\tau^m_j(u) = \enc(\Phi(u,j-1))\). Since \(|\Phi(u,j-1) - \Phi(u,j)| \leq 1\) and \(|\Phi(u,j) - \Phi(v,j)| \leq 1\) by Coherence~\ref{coh:4}, we have \(|\Phi(u,j-1) - \Phi(v,j)| \leq 2\), so \(\tau^m_j(u) = \enc(\Phi(v,j))\) forces \(\Phi(u,j-1) = \Phi(v,j)\) and hence \(\kappa^m_j(u) = H_{\Phi(v,j)}(u)\). Since the traffic-light condition already implies the snapshot condition, both conditions hold if and only if \(\ALIGNED_j(v)\) holds.

        \item By Coherence~\ref{coh:1}, \(\kappa^c_j(v) = H_{\Phi(v,j)}(v)\), so \(\neg\mathrm{Hlt}(\kappa^c_j(v))\) holds if and only if \(\mathrm{Hlt}(H_{\Phi(v,j)}(v)) = 0\). The claim follows from the definition of \(\EAGER\) and item~\ref{pred:behind}. \qedhere
      \end{enumerate}
    \end{proof}

    To prove the coherence lemma we establish a helper lemma: when a vertex is aligned, applying \(\mathrm{CMB}\) correctly produces the next snapshot.

    \begin{lemma}\label{lem:aligned-messages}
      Let \(\Phi\) be a correspondence coherent until step \(j\). If \(v\) is aligned at \(C_j\), then \(\mathrm{CMB}(\kappa^c_j(v), \hat{\kappa^m_j}(v)) = H_{\Phi(v,j)+1}(v)\).
    \end{lemma}
    \begin{proof}
      Since \(v\) is aligned at \(C_j\), Lemma~\ref{lem:predicates} gives \(\kappa^m_j(u) = H_{\Phi(v,j)}(u)\) for every neighbour \(u\). Substituting into the definition of \(\hat{\kappa^m_j}(v)\) gives
      \begin{align}
        \hat{\kappa^m_j}(v) &= \mathrm{AGG}\left(\mset{H_{\Phi(v,j)}(u) \mid u \in \nbr{G}{v}}\right).
      \end{align}
      By Coherence~\ref{coh:1}, \(\kappa^c_j(v) = H_{\Phi(v,j)}(v)\). Substituting both equalities into the definition of \(\mathcal{H}\)'s layer \(L\) gives
      \begin{align}
        \mathrm{CMB}(\kappa^c_j(v), \hat{\kappa^m_j}(v))
        &= L(H_{\Phi(v,j)})(v) \\
        &= H_{\Phi(v,j)+1}(v),
      \end{align}
      where the last equality is by definition of the halting run.
    \end{proof}

    \begin{proof}[Proof of Lemma~\ref{lem:coherence}]
      We proceed by induction on \(j\). Define \(\Phi(v, 0) \coloneqq 0\) for all \(v\). By the definition of \(\mathrm{In}'\), \(C_0(v) = (H_0(v),\, \enc(0),\, H_0(v),\, \enc(0))\), so Coherence~\ref{coh:1}--\ref{coh:4} hold at \(j = 0\).

      For the inductive step, suppose \(\Phi\) is coherent until step \(j\) and set \(\Phi(v, j+1) \coloneqq \Phi(v,j) + 1\) if \(v\) advances at \(C_j\), and \(\Phi(v, j+1) \coloneqq \Phi(v,j)\) otherwise. We verify Coherence~\ref{coh:1}--\ref{coh:4} at \(C_{j+1}\). For Coherence~\ref{coh:1} and~\ref{coh:2} we distinguish two cases.

      \begin{enumerate}
        \item If \(v\) advances~\eqref{eq:advancing}, then \(v\) is aligned at \(C_j\) and \(\Phi(v, j+1) = \Phi(v,j) + 1\). By Lemma~\ref{lem:aligned-messages} and Coherence~\ref{coh:2} at \(C_j\),
          \begin{align}
            \kappa^c_{j+1}(v) &= \mathrm{CMB}(\kappa^c_j(v), \hat{\kappa^m_j}(v)) = H_{\Phi(v,j)+1}(v), \\
            \tau^c_{j+1}(v) &= \adv(\tau^c_j(v)) = \enc(\Phi(v,j)+1).
          \end{align}

        \item If \(v\) waits~\eqref{eq:waiting}, then \(\Phi(v, j+1) = \Phi(v,j)\). By Coherence~\ref{coh:1} and~\ref{coh:2} at \(C_j\),
          \begin{align}
            \kappa^c_{j+1}(v) &= H_{\Phi(v,j)}(v), \\
            \tau^c_{j+1}(v) &= \enc(\Phi(v,j)).
          \end{align}
      \end{enumerate}
      Coherence~\ref{coh:3} at \(C_{j+1}\) holds in both cases since both branches of~\eqref{eq:comb-cases} set \(\kappa^m_{j+1}(v) = \kappa^c_j(v)\) and \(\tau^m_{j+1}(v) = \tau^c_j(v)\).

      It remains to verify Coherence~\ref{coh:4}. For any edge \((v, w)\), if both endpoints advance or neither does, the gap is unchanged. Suppose exactly one advances; call it \(a\) and the other \(b\). Then \(|\Phi(a, j+1) - \Phi(b, j+1)| = |\Phi(a,j) + 1 - \Phi(b,j)|\), which exceeds \(1\) only if \(\Phi(b,j) = \Phi(a,j) - 1\). At \(C_j\), \(a\) advances and is therefore aligned, so by Lemma~\ref{lem:predicates}, \(\tau^m_j(b) = \enc(\Phi(a,j))\). By Coherence~\ref{coh:3} and~\ref{coh:2}, \(\tau^m_j(b) = \tau^c_{j-1}(b) = \enc(\Phi(b,j-1))\), and \(\Phi(b,j-1) \in \{\Phi(b,j)-1,\, \Phi(b,j)\} = \{\Phi(a,j)-2,\, \Phi(a,j)-1\}\). Since \(|\Phi(b,j-1) - \Phi(a,j)| \leq 2 < 3\), we have \(\enc(\Phi(b,j-1)) \neq \enc(\Phi(a,j))\), a contradiction.
    \end{proof}

    The completeness proof relies on two further helper lemmas. Lemma~\ref{lem:behind-aligned} shows that behind vertices are always aligned and therefore advance immediately. Lemma~\ref{lem:upper-bound} shows that no vertex simulates beyond \(k_\Gamma\).

    \begin{lemma}\label{lem:behind-aligned}
      For every \(j \in \Nat\) and every vertex \(v\), if \(v\) is behind at \(C_j\), then \(v\) is aligned at \(C_j\).
    \end{lemma}
    \begin{proof}
      We proceed by strong induction on \(j\). At \(j = 0\), \(\Phi(u, 0) = 0\) for every vertex \(u\), so no vertex is behind. For \(j \geq 1\), suppose \(v\) is behind at converging state \(C_j\) with \(\Phi(v,j) = m\). Some neighbour \(u\) has \(\Phi(u,j) = m + 1\) by Lemma~\ref{lem:predicates}. We show that \(\Phi(w, j-1) = m\) for every neighbour \(w\) of \(v\), which implies that \(v\) is aligned at \(C_j\).

      First, we show that \(\Phi(v, j-1) = m\). Since \(\Phi(u, \cdot)\) increases by at most \(1\) per step, \(\Phi(u, j-1) \in \{m, m+1\}\). If \(\Phi(u, j-1) = m+1\), then Coherence~\ref{coh:4} and \(\Phi(v, j) = m\) give \(\Phi(v, j-1) = m\), so \(v\) is behind at \(C_{j-1}\). Since \(v\) is behind, by the inductive hypothesis \(v\) is aligned, and by the definition of eagerness, \(v\) is eager. Hence \(v\) advances, so \(\Phi(v, j) = m + 1\), a contradiction. Hence \(\Phi(u, j-1) = m\), and since \(\Phi(u,j) = m+1\), vertex \(u\) advanced from \(C_{j-1}\) to \(C_j\), requiring that \(u\) is aligned at \(C_{j-1}\). Since \(u\) is aligned at \(C_{j-1}\) and \(v\) is a neighbour of \(u\), we have \(\tau^m_{j-1}(v) = \enc(m)\).
      For \(j = 1\), the definition of \(\mathrm{In}'\) gives \(\tau^m_0(v) = \enc(0)\), so \(m = 0\) and \(\Phi(v, 0) = 0 = m\). For \(j \geq 2\), Coherence~\ref{coh:3} gives \(\tau^m_{j-1}(v) = \tau^c_{j-2}(v)\), so \(\tau^c_{j-2}(v) = \enc(m)\). By Coherence~\ref{coh:2}, \(\Phi(v, j-2) = m\), and since \(m = \Phi(v, j-2) \leq \Phi(v, j-1) \leq \Phi(v, j) = m\), we get \(\Phi(v, j-1) = m\).

      Second, we show that \(\Phi(w, j-1) = m\) for every neighbour \(w\) of \(v\). For any neighbour \(w\) of \(v\), Coherence~\ref{coh:4} at converging state \(C_{j-1}\) gives \(\Phi(w, j-1) \in \{m-1, m, m+1\}\). We show \(\Phi(w, j-1) = m\) by eliminating the other two possibilities.
      \begin{enumerate}
        \item If \(\Phi(w, j-1) = m + 1\), then \(v\) is behind at \(C_{j-1}\). Since \(v\) is behind, by the inductive hypothesis \(v\) is aligned, and by the definition of eagerness, \(v\) is eager. Hence \(v\) advances, contradicting \(\Phi(v, j) = m\).
        \item If \(\Phi(w, j-1) = m - 1\), then \(m \geq 1\). By its definition, \(j \geq \Phi(v, j)\), so \(j \geq \Phi(u, j) = m + 1 \geq 2\). Coherence~\ref{coh:4} at converging state \(C_{j-2}\) with \(\Phi(v, j-2) = m\) gives \(\Phi(w, j-2) \in \{m-1, m, m+1\}\), and since \(\Phi(w, j-2) \leq \Phi(w, j-1) = m - 1\) by definition, \(\Phi(w, j-2) = m - 1\). Hence \(w\) is behind at \(C_{j-2}\). Since \(w\) is behind, by the inductive hypothesis \(w\) is aligned, and by the definition of eagerness, \(w\) is eager. Hence \(w\) advances, so \(\Phi(w, j-1) \geq m\), a contradiction.
      \end{enumerate}

      Hence \(\Phi(w, j-1) = m\) for all neighbours \(w\). By Coherence~\ref{coh:2}, \(\tau^c_{j-1}(w) = \enc(m)\), and by Coherence~\ref{coh:3}, \(\tau^m_j(w) = \tau^c_{j-1}(w) = \enc(m) = \tau^c_j(v)\). Therefore \(v\) is aligned at \(C_j\).
    \end{proof}

    \begin{lemma}\label{lem:upper-bound}
      For every vertex \(v \in N_\Gamma\) and every \(j \in \Nat\), \(\Phi(v,j) \leq k_\Gamma\).
    \end{lemma}
    \begin{proof}
      Suppose for contradiction that \(\Phi(v, j) > k_\Gamma\) for some vertex \(v\) and some \(j \in \Nat\), and let \(j\) be minimal. Then \(\Phi(v, j-1) = k_\Gamma\) and \(v\) advances at \(C_{j-1}\), which requires \(v\) to be eager. By Coherence~\ref{coh:1}, \(\kappa^c_{j-1}(v) = H_{k_\Gamma}(v)\), and \(\mathrm{Hlt}(H_{k_\Gamma}(v))\) holds by definition of \(k_\Gamma\), so \(v\) is eager only if it is behind. By Lemma~\ref{lem:predicates}, being behind requires a neighbour \(u\) with \(\Phi(u, j-1) > k_\Gamma\), but \(\Phi(u, j-1) \leq k_\Gamma\) for all \(u\) by the choice of \(j\), a contradiction.
    \end{proof}

    \begin{proof}[Proof of Lemma~\ref{lem:completeness}]
      We show that if \(\Phi(v,j) < k_\Gamma\) for some \(v \in N_\Gamma\), then some vertex advances at \(C_j\) or \(C_{j+1}\). Let \(m = \min\{\Phi(v,j) : v \in N_\Gamma\} < k_\Gamma\). We identify a vertex \(v\) with \(\Phi(v,j) = m\) that is eager. If every vertex in \(N_\Gamma\) has index \(m\), then \(m < k_\Gamma\) implies some \(v\) satisfies \(\mathrm{Hlt}(H_m(v)) = 0\), so \(v\) is eager by Lemma~\ref{lem:predicates}. Otherwise, connectedness and Coherence~\ref{coh:4} yield a vertex \(v\) at index \(m\) with a neighbour at \(m + 1\), so \(v\) is behind and hence eager by Lemma~\ref{lem:predicates}.

      If \(v\) is behind at \(C_j\), Lemma~\ref{lem:behind-aligned} gives that \(v\) is aligned at \(C_j\), so \(v\) advances at \(C_j\). Otherwise, every vertex has index \(m\), and \(v\) is eager because \(\mathrm{Hlt}(H_m(v)) = 0\). If \(v\) is aligned at \(C_j\), then \(v\) advances at \(C_j\). Suppose \(v\) is not aligned at \(C_j\). Since \(v\) does not advance, \(\tau^c_{j+1}(v) = \enc(m)\) and \(\kappa^c_{j+1}(v) = \kappa^c_j(v)\). Since every neighbour \(w\) has \(\Phi(w,j) = m\), Coherence~\ref{coh:2} gives \(\tau^c_j(w) = \enc(m)\), and by Coherence~\ref{coh:3}, \(\tau^m_{j+1}(w) = \tau^c_j(w) = \enc(m)\), so \(v\) is aligned at \(C_{j+1}\). Since \(\kappa^c_{j+1}(v) = \kappa^c_j(v)\), \(\mathrm{Hlt}(H_m(v)) = 0\) persists, so \(v\) is eager at \(C_{j+1}\) and advances at \(C_{j+1}\).

      By Lemma~\ref{lem:upper-bound}, \(\Phi(v,j) \leq k_\Gamma\) for all \(v \in N_\Gamma\) and \(j \in \Nat\). Since \(\Phi\) is non-decreasing per vertex, and some vertex advances whenever \(\Phi(v,j) < k_\Gamma\) for some \(v\), we obtain \(\Phi(v, j') = k_\Gamma\) for all \(v \in N_\Gamma\) after finitely many steps.
    \end{proof}

    Once \(\Phi(v, j') = k_\Gamma\) for all \(v\), no vertex is behind or eager, so every vertex waits and the configuration stabilises after one additional step.

    \begin{proof}[Proof of Lemma~\ref{lem:convergence}]
      By Lemma~\ref{lem:completeness}, \(\Phi(v, j') = k_\Gamma\) for all \(v\). By Lemma~\ref{lem:upper-bound}, \(\Phi(v,j) \leq k_\Gamma\) for all \(j\). Since \(\Phi\) is also non-decreasing, \(\Phi(v, j) = k_\Gamma\) for all \(j \geq j'\). Since no neighbour has a different index, no vertex is behind by Lemma~\ref{lem:predicates}. By Coherence~\ref{coh:1}, \(\kappa^c_j(v) = H_{k_\Gamma}(v)\), and \(\mathrm{Hlt}(H_{k_\Gamma}(v))\) holds by definition of \(k_\Gamma\), so no vertex is eager. Since advancing~\eqref{eq:advancing} requires eagerness, every vertex waits~\eqref{eq:waiting} at \(C_j\) for every \(j \geq j'\), which preserves \(\kappa^c\) and \(\tau^c\). By Coherence~\ref{coh:3}, \(\kappa^m_j(v) = \kappa^c_{j-1}(v)\) and \(\tau^m_j(v) = \tau^c_{j-1}(v)\) for \(j \geq j' + 1\), and since \(\kappa^c\) and \(\tau^c\) are preserved, \(\kappa^m_j(v) = \kappa^c_j(v)\) and \(\tau^m_j(v) = \tau^c_j(v)\). All four components of \(\mathcal{C}_{j+1}(v)\) therefore equal those of \(\mathcal{C}_j(v)\).
    \end{proof}

    The correctness proof connects the converging run's output to the halting RGNN's output via Coherence~\ref{coh:1} and invariance under graded bisimulation.

    \begin{proof}[Proof of Lemma~\ref{lem:output}]
      By Lemma~\ref{lem:completeness} and Coherence~\ref{coh:1}, \(\kappa^c_{j'}(v) = H_{k_\Gamma}(v)\). By Lemma~\ref{lem:convergence}, \(\mathcal{C}_{j+1}(v) = \mathcal{C}_j(v)\) for \(j \geq j' + 1\). Since all vertices share index \(k_\Gamma\) at \(C_{j'}\), no vertex is behind by Lemma~\ref{lem:predicates}, and \(\mathrm{Hlt}(H_{k_\Gamma}(v))\) holds by definition of \(k_\Gamma\), so no vertex is eager and none advances. Hence \(\kappa^c_{j'+1}(v) = \kappa^c_{j'}(v)\), and \(\kappa^c_j(v) = H_{k_\Gamma}(v)\) for all \(j \geq j'\), and hence \(\mathrm{Out}'(\mathcal{C}_j(v)) = \mathrm{Out}(H_{k_\Gamma}(v))\).

      It remains to show that \(\mathrm{Out}(H_{k_\Gamma}(v)) = \mathcal{H}(G)(v)\).
      As \(\Gamma\) is a connected component, the relation $Z = \{ (v,v) \mid v \in N_\Gamma\}$ is a graded bisimulation between \(G\) and \(\Gamma\). Writing \(H^\Gamma_j(v)\) for the feature vectors of the run on \(\Gamma\), AC-layers preserve graded bisimulations, so \(H_j(v) = H^\Gamma_j(v)\) for all \(j\). In particular, \(\mathrm{Hlt}(H^\Gamma_{k_\Gamma}(v))\) holds for all \(v \in N_\Gamma\), so by definition \(\mathcal{H}(\Gamma)(v) = \mathrm{Out}(H^\Gamma_{k_\Gamma}(v)) = \mathrm{Out}(H_{k_\Gamma}(v))\). As \(\mathcal{H}\) is invariant under graded bisimulation, \(\mathcal{H}(G)(v) = \mathcal{H}(\Gamma)(v) = \mathrm{Out}(H_{k_\Gamma}(v))\).
    \end{proof}

    \begin{proof}[Proof of Proposition~\ref{prop:halting-to-converging}]
      Let \(\mathcal{C}\) be the RGNN constructed in Section~\ref{sec:construction}. By Lemma~\ref{lem:convergence}, the configuration of every vertex stabilises, so \(\mathcal{C}\) is a converging RGNN. By Lemma~\ref{lem:output}, the converged output at each vertex \(v\) equals \(\mathcal{H}(G)(v)\). The simplicity claim follows from Lemma~\ref{lem:simplicity}.
    \end{proof}

    \section{Related Work}\label{sec:related}

    The recurrent models of Scarselli et al.~\shortcite{Scarselli-2009-The-Graph-Neural}, Pflueger et al.~\shortcite{Pflueger-2024-Recurrent-Graph-Neural}, and Bollen et al.~\shortcite{Bollen-2025-Halting-Recurrent-GNNs} are discussed in Section~\ref{sec:rgnn}. Here we discuss two additional recurrent models with different output semantics.

    Ahvonen et al.~\shortcite{Ahvonen-2024-Logical-characterizations-of} study a visiting-acceptance semantics: a vertex is classified as true if its feature vector visits a designated accepting set at least once during the infinite run. They show that their RGNNs over reals capture countable disjunctions of GML formulas (\(\omega\)GML), while their RGNNs over floats capture the graded modal substitution calculus. As noted by Ahvonen et al., \(\omega\)GML and μGML are orthogonal in expressivity: visiting-acceptance can express properties not in μGML, such as the centre-point property, while μGML can express properties not in \(\omega\)GML, such as the non-reachability property~\cite{Kuusisto-2013-Modal-Logic-and}.

    Rosenbluth and Grohe~\shortcite{Rosenbluth-2025-Repetition-Makes-Perfect} study RGNNs where each vertex has a completion bit; the output is determined by the first feature vector where this bit is set. If given the graph size as input, these networks can compute any computable vertex function invariant under colour refinement. Since their model receives the graph size as additional input, it operates under different assumptions and is not directly comparable to ours.

    \section{Discussion}\label{sec:discussion}

    We have shown that, over undirected graphs, converging RGNNs and graded-bisimulation-invariant halting RGNNs express exactly the same vertex classifiers. Combined with the result of Bollen et al.\ that every μGML formula is expressible by a halting RGNN, yielding Corollary~\ref{cor:mu-gml}, and assuming the finitary Janin-Walukiewicz theorem extends to graded bisimulation, this would give a tight characterisation within MSO: the μGML-definable classifiers are exactly those that are both MSO-definable and expressible by a converging or output-converging RGNN.

    The restriction to MSO in this characterisation is necessary. For instance, the classifier that checks whether a vertex has equally many green neighbours as red neighbours is invariant under graded bisimulation and expressible by a converging RGNN (in a single iteration), but is not MSO-definable. A full logical characterisation of converging RGNN classifiers, beyond the MSO-definable fragment, remains open.

    Our equivalence is stated for undirected graphs. Given a halting RGNN with complete run \(H_0, \ldots, H_k\) on input \(G\), each connected component \(\Gamma\) converges independently: writing \(k_\Gamma\) for the first index at which all vertices in \(\Gamma\) satisfy \(\mathrm{Hlt}\), we may have \(k_\Gamma < k\), in which case no vertex in \(\Gamma\) ever simulates the halting states \(H_i\) for \(k_\Gamma < i \leq k\).

    The correctness argument relies on two properties: because our graphs are undirected, all vertices within a component converge at the same step \(k_\Gamma\), and because our RGNNs are invariant under graded bisimulation, \(\mathrm{Out}(H_{k_\Gamma}(v)) = \mathrm{Out}(H_k(v))\) for all vertices \(v\) in \(\Gamma\).

    On directed graphs, the first property does not hold. A vertex \(v\) can converge at some earlier state \(H_{k'}\), and vertices that still depend on it have no way to communicate to \(v\) that it has fallen behind. Even if \(\mathrm{Out}(H_{k'}(v)) = \mathrm{Out}(H_k(v))\), vertices that \(v\) sends messages to will never receive a message from \(v\) that reflects the halting state \(H_{k'+1}\) or later. By Lemma~\ref{lem:aligned-messages}, they will need these intermediate states to calculate the correct next snapshot, so the simulation will fail. Hence, whether some halting RGNNs can be faithfully simulated by converging RGNNs on directed graphs remains open.

    The halting-to-converging construction uses summation for the traffic light components, but this choice is not essential. Any aggregation function that preserves the support of the multiset of neighbours' traffic lights suffices, since this is all that \(\BEHIND\) and \(\ALIGNED\) require. Both mean and componentwise maximum satisfy this property: with one-hot traffic light vectors, each component of the aggregated result is nonzero if and only if at least one neighbour holds that value. If one were to redefine simple RGNNs with a different aggregation function, however, the simplicity proof becomes more sensitive. The proof converts aggregated traffic light norms into \(\{0, 1\}\) indicators via \(\min(x, 1) = x - \mathrm{ReLU}(x - 1)\), which requires these values to be positive integers. Summation yields neighbour counts, so this holds; componentwise maximum on one-hot vectors produces values in \(\{0, 1\}\) directly, so the same argument applies. Mean aggregation, however, produces rationals: a single ahead neighbour among three yields \(1/3\), not \(1\). Since no fixed simple function maps every positive real to \(1\) and \(0\) to \(0\), the thresholding step fails. Whether the simplicity result can be recovered for mean aggregation remains open.


\begin{thebibliography}{}

      \bibitem[\protect\citeauthoryear{Ahvonen \bgroup et al\mbox.\egroup
      }{2024}]{Ahvonen-2024-Logical-characterizations-of}
      Ahvonen, V.; Heiman, D.; Kuusisto, A.; and Lutz, C.
      \newblock 2024.
      \newblock Logical characterizations of recurrent graph neural networks with
      reals and floats.
      \newblock In Globerson, A.; Mackey, L.; Belgrave, D.; Fan, A.; Paquet, U.;
      Tomczak, J.; and Zhang, C., eds., {\em Advances in Neural Information
      Processing Systems}, volume~37,  104205--104249.
      \newblock Curran Associates, Inc.

      \bibitem[\protect\citeauthoryear{Ahvonen, Heiman, and
      Kuusisto}{2025}]{Ahvonen-2025-Graph-neural-networks}
      Ahvonen, V.; Heiman, D.; and Kuusisto, A.
      \newblock 2025.
      \newblock Graph neural networks and mso.

      \bibitem[\protect\citeauthoryear{Barceló \bgroup et al\mbox.\egroup
      }{2020}]{Barceló-2020-The-Logical-Expressiveness}
      Barceló, P.; Kostylev, E.~V.; Monet, M.; Pérez, J.; Reutter, J.; and Silva,
      J.~P.
      \newblock 2020.
      \newblock The logical expressiveness of graph neural networks.
      \newblock In {\em International Conference on Learning Representations}.

      \bibitem[\protect\citeauthoryear{Bollen \bgroup et al\mbox.\egroup
      }{2025}]{Bollen-2025-Halting-Recurrent-GNNs}
      Bollen, J.; Van~den Bussche, J.; Vansummeren, S.; and Virtema, J.
      \newblock 2025.
      \newblock {Halting Recurrent GNNs and the Graded mu-Calculus}.
      \newblock In {\em {Proceedings of the 22nd International Conference on
      Principles of Knowledge Representation and Reasoning}},  175--184.

      \bibitem[\protect\citeauthoryear{Colcombet, Doumane, and
      Kuperberg}{2025}]{Colcombet-2025-Tree-Algebras-and}
      Colcombet, T.; Doumane, A.; and Kuperberg, D.
      \newblock 2025.
      \newblock {Tree Algebras and Bisimulation-Invariant MSO on Finite Graphs}.
      \newblock In Censor-Hillel, K.; Grandoni, F.; Ouaknine, J.; and Puppis, G.,
      eds., {\em 52nd International Colloquium on Automata, Languages, and
      Programming (ICALP 2025)}, volume 334 of {\em Leibniz International
      Proceedings in Informatics (LIPIcs)},  152:1--152:16.
      \newblock Dagstuhl, Germany: Schloss Dagstuhl -- Leibniz-Zentrum f{\"u}r
      Informatik.

      \bibitem[\protect\citeauthoryear{Corso \bgroup et al\mbox.\egroup
      }{2020}]{Corso-2020-Principal-Neighbourhood-Aggregation}
      Corso, G.; Cavalleri, L.; Beaini, D.; Li\`{o}, P.; and Veli\v{c}kovi\'{c}, P.
      \newblock 2020.
      \newblock Principal neighbourhood aggregation for graph nets.
      \newblock In Larochelle, H.; Ranzato, M.; Hadsell, R.; Balcan, M.; and Lin, H.,
      eds., {\em Advances in Neural Information Processing Systems}, volume~33,
      13260--13271.
      \newblock Curran Associates, Inc.

      \bibitem[\protect\citeauthoryear{Gilmer \bgroup et al\mbox.\egroup
      }{2017}]{Gilmer-2017-Neural-message-passing}
      Gilmer, J.; Schoenholz, S.~S.; Riley, P.~F.; Vinyals, O.; and Dahl, G.~E.
      \newblock 2017.
      \newblock Neural message passing for quantum chemistry.
      \newblock In {\em Proceedings of the 34th International Conference on Machine
      Learning - Volume 70}, ICML'17,  1263–1272.
      \newblock JMLR.org.

      \bibitem[\protect\citeauthoryear{Grohe}{2021}]{Grohe-2021-The-Logic-of}
      Grohe, M.
      \newblock 2021.
      \newblock The logic of graph neural networks.
      \newblock In {\em 2021 36th Annual ACM/IEEE Symposium on Logic in Computer
      Science (LICS)},  1--17.

      \bibitem[\protect\citeauthoryear{Hamilton, Ying, and
      Leskovec}{2017}]{Hamilton-2017-Inductive-Representation-Learning}
      Hamilton, W.; Ying, Z.; and Leskovec, J.
      \newblock 2017.
      \newblock Inductive representation learning on large graphs.
      \newblock In Guyon, I.; Luxburg, U.~V.; Bengio, S.; Wallach, H.; Fergus, R.;
      Vishwanathan, S.; and Garnett, R., eds., {\em Advances in Neural Information
      Processing Systems}, volume~30.
      \newblock Curran Associates, Inc.

      \bibitem[\protect\citeauthoryear{Hamilton}{2020}]{Hamilton-2020-Graph-Representation}
      Hamilton, W.~L.
      \newblock 2020.
      \newblock Graph representation learning.
      \newblock {\em Synthesis Lectures on Artificial Intelligence and Machine
      Learning} 14(3):1--159.

      \bibitem[\protect\citeauthoryear{Janin and
      Lenzi}{2001}]{Janin-2001-Relating-Levels-of}
      Janin, D., and Lenzi, G.
      \newblock 2001.
      \newblock Relating levels of the mu-calculus hierarchy and levels of the
      monadic hierarchy.
      \newblock In {\em 16th Annual {IEEE} Symposium on Logic in Computer Science,
      Boston, Massachusetts, USA, June 16-19, 2001, Proceedings},  347--356.
      \newblock {IEEE} Computer Society.

      \bibitem[\protect\citeauthoryear{Kipf and
      Welling}{2017}]{Kipf-2017-Semi-Supervised-Classification}
      Kipf, T.~N., and Welling, M.
      \newblock 2017.
      \newblock Semi-supervised classification with graph convolutional networks.
      \newblock In {\em International Conference on Learning Representations}.

      \bibitem[\protect\citeauthoryear{Kuusisto}{2013}]{Kuusisto-2013-Modal-Logic-and}
      Kuusisto, A.
      \newblock 2013.
      \newblock {Modal Logic and Distributed Message Passing Automata}.
      \newblock In Ronchi Della~Rocca, S., ed., {\em Computer Science Logic 2013 (CSL
      2013)}, volume~23 of {\em Leibniz International Proceedings in Informatics
      (LIPIcs)},  452--468.
      \newblock Dagstuhl, Germany: Schloss Dagstuhl -- Leibniz-Zentrum f{\"u}r
      Informatik.

      \bibitem[\protect\citeauthoryear{Morris \bgroup et al\mbox.\egroup
      }{2019}]{Morris-2019-Weisfeiler-and-Leman}
      Morris, C.; Ritzert, M.; Fey, M.; Hamilton, W.~L.; Lenssen, J.~E.; Rattan, G.;
      and Grohe, M.
      \newblock 2019.
      \newblock Weisfeiler and leman go neural: Higher-order graph neural networks.
      \newblock {\em Proceedings of the AAAI Conference on Artificial Intelligence}
      33(01):4602--4609.

      \bibitem[\protect\citeauthoryear{Pflueger, Tena~Cucala, and
      Kostylev}{2024}]{Pflueger-2024-Recurrent-Graph-Neural}
      Pflueger, M.; Tena~Cucala, D.; and Kostylev, E.~V.
      \newblock 2024.
      \newblock Recurrent graph neural networks and their connections to bisimulation
      and logic.
      \newblock {\em Proceedings of the AAAI Conference on Artificial Intelligence}
      38(13):14608--14616.

      \bibitem[\protect\citeauthoryear{Rosenbluth and
      Grohe}{2025}]{Rosenbluth-2025-Repetition-Makes-Perfect}
      Rosenbluth, E., and Grohe, M.
      \newblock 2025.
      \newblock Repetition makes perfect: Recurrent graph neural networks match
      message-passing limit.

      \bibitem[\protect\citeauthoryear{Scarselli \bgroup et al\mbox.\egroup
      }{2009}]{Scarselli-2009-The-Graph-Neural}
      Scarselli, F.; Gori, M.; Tsoi, A.~C.; Hagenbuchner, M.; and Monfardini, G.
      \newblock 2009.
      \newblock The graph neural network model.
      \newblock {\em IEEE Transactions on Neural Networks} 20(1):61--80.

      \bibitem[\protect\citeauthoryear{Walukiewicz}{2002}]{Walukiewicz-2002-Monadic-second-order}
      Walukiewicz, I.
      \newblock 2002.
      \newblock Monadic second-order logic on tree-like structures.
      \newblock {\em Theor. Comput. Sci.} 275(1–2):311–346.

      \bibitem[\protect\citeauthoryear{Xu \bgroup et al\mbox.\egroup
      }{2019}]{Xu-2019-How-Powerful-are}
      Xu, K.; Hu, W.; Leskovec, J.; and Jegelka, S.
      \newblock 2019.
      \newblock How powerful are graph neural networks?
      \newblock In {\em International Conference on Learning Representations}.

    \end{thebibliography}
    \end{document}